\PassOptionsToPackage{nameinlink}{cleveref}
\documentclass[12pt,cleveref]{colt2024}
\def\Comments{0} 

\title{Learning Neural Networks with Sparse Activations}
\usepackage{times}

\coltauthor{%
 \Name{Pranjal Awasthi} \Email{pranjalawasthi@google.com}\\
 \addr Google Research
 \AND
 \Name{Nishanth Dikkala} \Email{nishanthd@google.com}\\
 \addr Google Research
 \AND
 \Name{Pritish Kamath} \Email{pritish@alum.mit.edu}\\
 \addr Google Research
 \AND
 \Name{Raghu Meka} \Email{raghum@cs.ucla.edu}\\
 \addr University of California, Los Angeles
}

\usepackage{amsfonts,amsmath,amssymb,bm,dsfont}
\usepackage{hyperref}
\usepackage{xcolor}
\usepackage{enumitem}

\newtheorem{claim}{Claim}[section]

\hypersetup{
	colorlinks=true,
	linkcolor=blue,
	citecolor=blue,
}

\definecolor{Gred}{RGB}{219, 50, 54}
\definecolor{Ggreen}{RGB}{60, 186, 84}
\definecolor{Gblue}{RGB}{72, 133, 237}
\definecolor{Gyellow}{RGB}{247, 178, 16}
\definecolor{ToCgreen}{RGB}{0, 128, 0}
\definecolor{myGold}{RGB}{231,141,20}
\definecolor{myBlue}{rgb}{0.19,0.41,.65}
\definecolor{myPurple}{RGB}{175,0,124}

\providecommand{\Comments}{1}
\ifnum\Comments=1
\usepackage[colorinlistoftodos,prependcaption,textsize=scriptsize]{todonotes}
\paperwidth=\dimexpr \paperwidth + 3cm\relax
\oddsidemargin=\dimexpr\oddsidemargin + 1.5cm\relax
\evensidemargin=\dimexpr\evensidemargin + 1.5cm\relax
\marginparwidth=\dimexpr\marginparwidth + 4cm\relax
\else
\usepackage[disable]{todonotes}
\fi
\newcommand{\mytodo}[1]{\ifnum\Comments=1{#1}\fi}

\newcommand{\nishanth}[1]{\todo[linecolor=Gblue,backgroundcolor=Gblue!25,bordercolor=Gblue]{Nishanth: #1}}

\newcommand{\pritish}[1]{\todo[linecolor=myGold,backgroundcolor=myGold!25,bordercolor=myGold]{Pritish: #1}}

\newcommand{\tableoftodos}{\ifnum\Comments=1 \listoftodos[Comments/To Do's] \fi}

\newcommand{\eps}{\varepsilon}

\newcommand{\relu}{\sigma} 

\newcommand{\ang}[1]{\left \langle #1 \right \rangle}

\newcommand{\prn}[1]{\left ( #1 \right )}
\newcommand{\sq}[1]{\left [ #1 \right ]}
\newcommand{\floor}[1]{\left \lfloor #1 \right \rfloor}

\newcommand{\Mod}[1]{\ (\mathrm{mod}\ #1)}

\def\ddefloop#1{\ifx\ddefloop#1\else\ddef{#1}\expandafter\ddefloop\fi}
\def\ddef#1{\expandafter\def\csname b#1\endcsname{\ensuremath{\bm #1}}}
\ddefloop abcdeghijklnopqrstuvwxyzABCDEFGHIJKLMNOPQRSTUVWXYZ \ddefloop  
\def\ddef#1{\expandafter\def\csname #1\endcsname{\ensuremath{\mathbb{#1}}}}
\ddefloop ABCDFGHIJKLMNOPQRSTUVWXYZ \ddefloop  
\DeclareMathOperator*{\E}{\mathop{\mathbb{E}}}
\def\ddef#1{\expandafter\def\csname c#1\endcsname{\ensuremath{\mathcal{#1}}}}
\ddefloop ABCDEFGHIJKLMNOPQRSTUVWXYZ \ddefloop
\def\ddef#1{\expandafter\def\csname h#1\endcsname{\ensuremath{\hat{#1}}}}
\ddefloop fghxyz \ddefloop

\newcommand{\AS}{\mathsf{AS}}
\newcommand{\NS}{\mathsf{NS}}
\newcommand{\poly}{\mathsf{poly}}
\newcommand{\sen}{\mathsf{sen}}
\newcommand{\sbit}{\{-1, 1\}}

\newcommand{\ignore}[1]{{}}

\begin{document}

\maketitle

\begin{abstract}
A core component present in many successful neural network architectures, is an MLP block of two fully connected layers with a non-linear activation in between. An intriguing phenomenon observed empirically, including in transformer architectures, is that, after training, the activations in the hidden layer of this MLP block tend to be extremely sparse on any given input. Unlike traditional forms of sparsity, where there are neurons/weights which can be deleted from the network, this form of {\em dynamic} activation sparsity appears to be harder to exploit to get more efficient networks. 

Motivated by this we initiate a formal study of PAC learnability of MLP layers that exhibit activation sparsity. We present a variety of results showing that such classes of functions do lead to provable computational and statistical advantages over their non-sparse counterparts. Our hope is that a better theoretical understanding of {\em sparsely activated} networks would lead to methods that can exploit activation sparsity in practice.
\end{abstract}


\section{Introduction}\label{sec:intro}

In recent years, transformer based deep neural networks \citep{vaswani2017attention} and the subsequent development of large language models have marked a paradigm shift in the fields of natural language processing and computer vision \citep{brown2020language, chowdhery2022palm, chen2022pali, dosovitskiy2020image}. These models have significantly improved performance across various tasks, setting new benchmarks and enabling previously unattainable breakthroughs. However, the computational cost of training and deploying these models, especially the largest variants, presents a significant challenge. A notable portion of these models' computational and parameter overhead is attributed to the Multi-Layer Perceptron (MLP) layers. These layers are integral to the transformer architecture, playing a crucial role in its ability to solve many different tasks.

Despite their efficacy, the resource-intensive nature of these models has spurred a wave of research focused on enhancing their efficiency \citep{banner2019post, frankle2018lottery, gholami2022survey, hinton2015distilling, anil2018large, harutyunyan2023supervision}. Among the various strategies explored for improving the inference efficiency of large transformers, attempting to sparsify the transformer is a promising approach.


A motivation for exploiting sparsity is rooted in an intriguing empirical observation made in recent works \citep{li2023lazy} regarding the behavior of MLP layers within large transformer models. Post-training, these layers tend to exhibit a high degree of sparsity in their activations; often each input activates as low as 3\% of the neurons in the MLP layers, suggesting a natural emergence of sparsity in activations. This leads to these MLP layers behaving like key-value lookups \citep{geva2020transformer}. The extremely low sparsity (3\%) suggests that there might be significant room to sparsify the MLP layers leading to both training and inference efficiency. In addition, such sparsity also helps with interpretability of transformers by disentangling neurons corresponding to distinct concepts \citep{elhage2022solu}. Moreover, through extensive ablation studies \cite{li2023lazy} observe that this phenomenon is highly prevalent. It occurs in convolutional networks (CNNs), as well as in vanilla fully connected feedforward networks. 

Despite the potential benefits, effectively harnessing dynamic sparsity has proven challenging. Although, there have been many recent efforts \citep{li2023lazy, grimaldi2023accelerating, liu2023dejavu, dong2023towards, csordas2023approximating, mirzadeh2023relu}, they have led to limited success. None of the approaches achieve speedups (either in training or in inference) anywhere close to the the potential factor of 33x that is suggested by 3\% sparsity.
Moreover, by explicitly enforcing sparsity via methods such as choosing only the top-$k$ activations, the quality of the model degrades in some cases. 

A key reason for the hardness in exploiting activation sparsity is that this form of sparsity is {\em dynamic} in nature and is input-dependent (i.e., not a fixed pattern). While each input example activates a small number of neurons, the overall sparsity pattern cannot be localized to a small subset of the model weights. For instance, the dynamic nature precludes the use of typical weight quantization or pruning based methods to exploit sparsity empirically. On the other hand, having a non-localized sparsity pattern is crucial in ensuring the model has rich expressiveness.


The above observations suggest that post-training, large transformer networks belong to an intriguing function class that is highly expressive yet exhibits high sparsity. Given the challenges in exploiting this behavior in practical settings, in this work, we initiate a theoretical study of the statistical and computational properties of such functions in the probably approximately correct (PAC) learning framework \citep{valiant1984theory}.


We introduce the class of {\em sparsely activated} MLPs. We focus on the case of depth-$1$ MLPs with $n$ input units and $s$ hidden units with the standard ReLU activations. We define the class $\cH_{n,s,k}$ as the class of depth-$1$ ReLU networks in $n$-dimensions with the promise that on each input in the support of the data distribution, at most $k$ of the $s$ hidden units are active:

\begin{definition}[Sparsely Activated Networks]
Let $\relu(\cdot)$ denote the $\mathsf{ReLU}$ activation, namely $\relu(z) := \max\{z, 0\}$. The class  $\cH_{n,s,k}$ consists of hypotheses of the form $h(x) = \sum_{j=1}^s u_j \relu(\ang{w_j, x} - b_j)$ with the property that for all $x$ in the support of the distribution, it holds that $|\{ j : \ang{w_j, x} - b_j > 0 \}| \le k$. 
    
\end{definition}

Note that this sparsity differs from {\em dead sparsity}, where some neurons are never active on any of the inputs, and consequently, can be deleted from the network without impacting its functionality. The form of dynamic sparsity we study can be crucial for the networks to be more expressive. We provide a couple of examples of useful functions represented using sparsely activated networks here:
\begin{itemize}[leftmargin=4mm]
    \item \textbf{Junta functions:} The class of functions on $n$ variables which depend on only a $p$-sized subset ($p < n$) of the variables is known as $p$-junta functions. Sparse parities are a canonical example of junta functions. We show in Theorem~\ref{thm:k=1-uniform-lb} that we can represent $\log(s)$-juntas using $\cH_{n,s,1}$.
    \item \textbf{Indexing function:} Consider the function $\mathsf{Index}_b : \{-1, 1\}^{b + 2^b} \to \{0, 1\}$, where $\mathsf{Index}_b(z)$ is the $x$-th bit of $y$ ($-1$ mapped to $0$), where $x$ is the integer represented by the first $b$ bits of $z$ in binary representation, and $y$ is the remaining $2^b$ bits vector. This can be represented as a $1$-sparse activation network of size $2^b$ (i.e., in $\mathcal{H}_{b+2^b, 2^b, 1}$):  $\mathsf{Index}_b((x,y)) = \sum_{\alpha \in \{-1, 1\}^b} \sigma(\langle{w_\alpha, z} \rangle - b + \frac12)$ where the first $b$ coordinates of $w_\alpha$ are $\alpha$ and the $\alpha$-th coordinate among the last $2^b$ coordinates is $\frac12$. On input $z = (x, y)$, only the neuron corresponding to $\alpha = x$ is activated, and the output is precisely $\frac12 y_x + \frac12$. 
\end{itemize}
In both the examples presented above, removing any of the $s$ neurons will change the functionality of the network. However, each weight vector $w_i$ is quite sparse. In Appendix~\ref{apx:san-without-deadsparsity}, we present an example of a sparsely activated network where even the weight vectors $w_i$ are not sparse. Hence, in general, it is not clear if sparsely activated networks can be represented with fewer neurons or sparse weight vectors.

In order to provide learning guarantees, we have to assume an upper bound on the {\em scale} of $u$, $w_j$'s and $b_j$'s. We will use the following natural scaling for the paper: 

\begin{definition}\label{def:inf-norm-bound}
Let $\cH_{n,s,k}^{W, B} \subseteq \cH_{n,s,k}$ consisting of $h$ given as $h(x) = \sum_{j=1}^s u_j \relu(\ang{w_j, x} - b_j)$, satisfying
$\|u\|_\infty \cdot \max_{j\in[s]} \|w_j\|_2 \le W$ and $\|u\|_\infty \cdot \max_{j\in[s]} |b_j| \le B$.
\end{definition}


We then consider the problem of learning sparsely activated networks efficiently. We consider the domain to be the Boolean hypercube $\cX = \{1,-1\}^n$ as a natural first-step and as a domain where sparsely activated networks can compute non-trivial functions.
The Boolean hypercube provides a setting where the function can be sparse everywhere in the domain while maintaining expressiveness; this appears harder in the continuous setting. For instance, if the inputs are Gaussian over $\mathbb{R}^n$, one likely needs the biases in the ReLU units to be very large to enforce $1$-sparsity. 
This suggests that, in the continuous domain, more non-standard distributions are likely necessary to obtain a rich class of functions which are sparse everywhere in the domain. Hence for theoretical simplicity we focus on functions on the Boolean hypercube.

Even with the sparsity assumption, the class $\cH_{n,s,1}$ is likely hard to learn in polynomial time (or even quasi-polynomial time) under an arbitrary distribution on the hypercube. In particular, we show that \emph{parities} on the hypercube on $k$ variables can be computed by $\cH_{k^2,2k,1}$, with coefficient vectors of norm at most $O(k)$. Thus, $\cH_{n,O(\sqrt{n}),1}$ need $2^{\Omega(\sqrt{n})}$ queries in the powerful Statistical Queries (SQ) model (see \Cref{sec:lb-uniform} for details). We also show cryptographic hardness results for learning $\cH_{n,s,1}$ under generic distributions on the hypercube. 

\begin{theorem}[Informal; see \Cref{sec:lb-uniform}]\label{th:lb-general}
    Any SQ algorithm for learning $\cH_{n,O(\sqrt{n}),1}^{O(n^{0.75}), O(n)}$ under arbitrary distributions over the hypercube either requires $2^{-\Omega(\sqrt{n})}$ tolerance or $2^{\Omega(\sqrt{n})}$ queries. 

    Assuming the hardness of \emph{learning with rounding} problem with polynomial modulus, there is no $\poly(n,s,W,B,1/\eps)$ run-time algorithm to $(\eps,\delta)$-PAC learn $\cH_{n,s,1}^{W, B}$. 
\end{theorem}

\paragraph{Learning under uniform distribution.}
Given the above hardness results, it is natural to consider distributional assumptions as is often done for related classes in learning theory (e.g., \cite{KlivansOS04,kane14average} etc.). Our main result is that when the input distribution is uniform over the $n$-dimensional hypercube, $\{1,-1\}^n$, the class $\cH^{W,B}_{n,s,k}$ can be learned in time $n^{\poly(k \log (ns))}$:

\begin{theorem}[Informal; see \Cref{thm:generalk-uniform-ub}]
\label{thm:k-uniform-ub}
There exists an $(\eps,\delta)$-PAC learning algorithm for $\cH_{n,s,k}^{W, B}$ with respect to the uniform distribution over $\{1,-1\}^n$ that has sample complexity and run-time $n^{\poly(k \log(ns)) / \eps^2} \log(1/\delta) / \eps$ (suppressing dependence on $W, B$).
\end{theorem}
As our learning algorithm works by performing linear regression over low-degree monomial basis (a.k.a. the \emph{low-degree algorithm}), the guarantees work even in the \emph{agnostic} or \emph{non-realizable} setting by standard arguments (e.g., \cite{KlivansOS04}). For simplicity, we focus on the realizable setting as the algorithm and analysis do not change for the agnostic case.

For sparsity $k=1$, the above run-time is $n^{O(\poly(\log(ns))/\eps^2)}$. As we showed above, $\cH_{n,s,1}$ can simulate juntas of size $\log_2 s$ over $n$ variables. Thus, a quasi-polynomial run-time is the best we can do under a widely believed conjecture on the hardness of learning juntas.

The guarantee above is in stark contrast to what is achievable for general one-layer size $s$ ReLU networks under the uniform distribution over the hypercube. One-layer size-$s$ networks can simulate parities on $\min(n,s)$ variables. They thus cannot be learned even under the uniform distribution on the hypercube by SQ algorithms with less than $2^{\Omega(\min(n,s))}$ queries. Further, even for non-SQ algorithms, as shown in \citep{ChenGKM22}, quasi-polynomial run-time with respect to the uniform distribution on the hypercube is impossible under widely studied cryptographic assumptions. 

The proof of \Cref{thm:k-uniform-ub} is via Fourier analysis and the \emph{low-degree algorithm}. The main ingredient is to show that the \emph{average-sensitivity} of functions in $\cH_{n,s,k}$ is at most $O(k^4 (\sqrt{n}\log (ns)))$. We then use this bound the \emph{noise-sensitivity} of functions in $\cH_{n,s,k}$. The latter implies the existence of a low-degree approximation by exploiting \cite{KlivansOS04} which is enough to obtain the theorem. See \Cref{sec:ub-uniform} for details.


\paragraph{Learning under general distributions.}
We also show that $\cH_{n,s,k}^{W,B}$ can be learnt under general distributions with smaller sample complexity than would be required without the sparsity condition, in the case when $s \gg kn$. In particular, we show the following.
\begin{theorem}[Informal; see \Cref{thm:general-dist-upper-bound}]
There exists an $(\eps,\delta)$-PAC learning algorithm for $\cH_{n,s,k}^{W, B}$ over $\{1,-1\}^n$ that has sample complexity $\widetilde{O}\left(ksn/\eps^2\right)$ (suppressing dependence on $W, B, \delta$).
\end{theorem}

By contrast, the class $\cH_{n,s,s}^{W,B}$ (that is, size-$s$ networks without activation sparsity) requires a sample complexity of $\Omega(s^2 / \eps^2)$.\pritish{Is this right?}
To prove the above, we provide a bound on the Rademacher complexity of the class $\cH_{n,s,k}^{W,B}$ that has an improved dependence on $s$.

Taken together, our results demonstrate that leveraging dynamic activation sparsity is theoretically possible for both computational and statistical benefits. We hope that further theoretical study of the class of sparsely activated networks could pave the way for more efficient training and inference methods for deep architectures, including transformer-based models where these sparsely activated networks have been observed to arise in practice.

\subsection{Related Work}

Our work is motivated by recent empirical observations on the extreme sparsity observed in the MLP layers of trained transformer models \citep{li2023lazy, shen2023study}. The works of \citet{li2023lazy, peng2023theoretical} propose theoretical explanations of why this phenomenon occurs. However, ours is the first work to formally study sparsely activated networks in the PAC learning setup and quantify their computational and statistical advantages. Motivated by the observation on sparsity, recent work has also studied the connections between the MLP layers and key-value memory lookups \citep{sukhbaatar2019augmenting, lample2019large, geva2020transformer}. 

There have also been recent works on designing networks with explicitly enforced sparsity structure. One such line of work concerns mixture of experts models \citep{shazeer2017outrageously, fedus2022switch} where each input is independently routed to one or two MLP blocks among a set of experts. An alternate way to enforce sparsity is to introduce a top-$k$ operation after each MLP layer that zeros out most of the activations \citep{csordas2023approximating, li2023lazy}. In particular, \citet{li2023lazy} propose a top-$k$ transformer along these lines. However, due to the top-$k$ operation being relatively slow on accelerator hardware, this technique does not yield wall-clock speedup for either training or inference. 

In another recent work \citet{liu2023dejavu} propose to train a small predictor network to predict the activated indices at each MLP layer. There has also been work to explore enforcing block sparsity constraints and weight tying in the model weights themselves \citep{dong2023towards}, as well as efforts to enforce static sparsity that is not input dependent \citep{frantar2023sparsegpt}. However such methods haven't been effective for language modeling via transformer models and have been much more successful in classification domains that have a small number of output labels.

 Significantly more attention has been given to sparsifying attention layer computation \citep{zaheer2020big, choromanski2020rethinking, wang2020linformer, gu2023mamba}. 
 Instead, our focus in this work here is understanding the sparsity behavior of the MLP layer.

\section{Preliminaries}\label{sec:prelims}

We consider the problem of learning real-valued functions over the input space $\cX = \sbit^n$, to small expected $\ell_2$-squared error, namely for the underlying distribution $\cD$ over $(x, y) \in \cX \times \R$, our goal is the minimize the population loss of a predictor $f : \cX \to \R$ given as $\cL_{\cD}(f) := \E_{(x,y) \sim \cD} \ell(f(x), y)$ where $\ell(\hy, y) := \frac12 (\hy - y)^2$. For any dataset $S \in (\cX \times \R)^*$, we denote the empirical loss as $\cL_S(f) := \frac{1}{|S|} \sum_{(x,y) \in S} \ell(f(x), y)$.

For any hypothesis class $\cH \subseteq \R^\cX$, we say that $\cD$ is $\cH$-realizable, if there exists $h^{\star} \in \cH$ such that $h^{\star}(x) = y$ holds with probability $1$ for $(x, y) \sim \cD$. Following the standard definition of {\em probably approximately correct} (PAC) learning \citep{valiant1984theory}, we say that a learning algorithm $\cA$ $(\eps, \delta)$-PAC learns $\cH$ with sample complexity $m(\eps, \delta)$ if for all $\cH$-realizable distributions $\cD$ over $\cX \times \R$, and for $S \sim \cD^{m(\eps, \delta)}$, it holds with probability at least $1-\delta$
that $\cL_\cD(\cA(S)) \le \eps$. We say that a learning algorithm $\cA$ $(\eps, \delta)$-PAC learns $\cH$ under distribution $\cP$ (over $\cX$) if the learning guarantee holds for all $\cH$-realizable $\cD$ with the marginal over $\cX$ being $\cP$. In particular, we use $\cU$ to denote the uniform distribution over $\cX$.

%

\subsection{Fourier Analysis and the Low-Degree Algorithm}\label{subsec:fourier}

Any function $f : \sbit^n \to \R$, has a unique Fourier representation given as $\sum_{T \subseteq [n]} \hat{f}(T) \chi_T(x)$ where $\chi_T(x) := \prod_{j \in T} x_i$.
The degree of $f$, denoted $\deg(f)$, is the largest $k$ such that $\hf(T) \ne 0$ for some $T$ with $|T| = k$.
The $\ell_2$ norm of $f$ under the uniform distribution is defined as $\|f\|_2^2 := \E_{x \sim \cU} f(x)^2$ \citep{o2014analysis}.

We define the $\ell_2$ sensitivity of $f$ at $x$ as $\sen_f(x) := \frac14\sum_{i \in [n]} (f(x) - f(x^{\oplus i}))^2$, where $x^{\oplus i}$ is $x$ with the $i$-th bit flipped; the scaling factor of $1/4$ means that for $f : \sbit^n \to \sbit$, sensitivity can be interpreted as $\sen_f(x) = |\{ i : f(x) \ne f(x^{\oplus i})\}|$.
The average $\ell_2^2$ sensitivity $\AS(f)$ is defined as $\E_{x \sim \cU} \left[ \sen_f(x) \right]$.
For any $x$, let $N_{\rho}(x)$ denote the distribution obtained by flipping each coordinate of $x$ with probability $(1-\rho)/2$.
The $\rho$-noise sensitivity of $f$ is $\NS_\rho(f) := \E_{x \sim \cU, y \sim N_{\rho}(x)} \frac14 (f(x) - f(y))^2$.

A connection between noise sensitivity and Fourier concentration was first observed in \cite{KlivansOS04}. We state this connection below, along with other basic facts about Fourier coefficients.

\begin{claim}\label{fact:fourier}[See \cite{KlivansOS04}]
The following properties hold for all $f : \sbit^n \to \R$:
\begin{itemize}[leftmargin=5mm,itemsep=1pt]
\item $\|f\|_2^2 = \sum_{T \subseteq [n]} \hf(T)^2$, and
\item $\NS_\rho(f) = \sum_{T \subseteq [n]} \frac12 (1 - \rho^{|T|}) \hf(T)^2$, and hence $\sum_{T : |T| > d} \hf(T)^2 \le 2 \cdot \NS_\rho(f) / (1 - \rho^d)$.
\end{itemize}
\end{claim}

We also need a bound on the average sensitivity of a single halfspace which is known to be $O(\sqrt{n})$. We require a more fine-grained version from \cite{kane14average} which quantifies the dependence on the bias of the halfspace.  

\begin{lemma}[\cite{kane14average}]\label{lm:ashalfspace}
    Let $g:\cX \rightarrow \{0,1\}$ be a halfspace: $g(x) = \mathds{1}\{\ang{w, x} \le b\}$ and $\E[g] = p$. Then, $\AS(g) = O(p \sqrt{n \log(1/p)})$.
\end{lemma}

\begin{proof}
Without loss of generality, we can assume that the coefficients of $w$ are positive. This makes $g$ a monotone function which is non-decreasing in each coordinate. Now, for $i \in [n]$, and $x \sim \cU$,
    $$\E[x_i g(x)] = \frac{1}{2}\sum_{x \in \cX}\left(x_ig(x) - x_i g(x^{\oplus i}) \right) = \E[(g(x) - g(x^{\oplus i}))^2],$$
where the second equality is due to the non-decreasing nature of $g$ and that $g(x)$ takes values in $\{0, 1\}$.
Therefore,
$$\textstyle \AS(g) = \frac{1}{4} \E_x\sq{g(x) \sum_{i=1}^n x_i},$$
the claim now follows from Lemma 6 of \cite{kane14average}.
\end{proof}

\paragraph{Low-degree algorithm.} We recall the standard {\em low-degree} algorithm and its guarantees for learning hypothesis classes that exhibit low-degree Fourier concentration (see e.g., \cite{KlivansOS04} for details).
For any hypothesis class $\cH \subseteq (\cX \to \R)$, let $C_{\cH} := \sup_{h \in \cH, x \in \cX} h(x)$.
\begin{lemma}\label{lem:low-degree-alg}
For hypothesis class $\cH \subseteq (\cX \to \R)$ such that $\sum_{T: |T| > d} \hh(T)^2 \le \eps$ for all $h \in \cH$, there exists an $(O(\eps), \delta)$-PAC learning algorithm for $\cH$ with $O(n^d C_{\cH}^2 \log(1/\delta)/\eps)$ sample and time complexity.
\end{lemma}

\noindent The algorithm operates by performing {\em polynomial regression}, that is, linear regression in the basis of monomials of degree at most $d$. The algorithm achieves the desired error because $g(x) := \sum_{T : |T| \le d} \hh(T) \chi_T(x)$ is such that $\|g - h\|_2^2 = \sum_{T : |T| > d} \hh(T)^2 \le \eps/2$, and hence there exists a good solution to the polynomial regression problem.\pritish{Cite Hsu-Kakade-Zhang?}

\section{Learning over Uniform Distribution}\label{sec:ub-uniform}

In this section we provide a learning algorithm for $k$-sparsely activated networks under the uniform distribution.

\begin{theorem}
\label{thm:generalk-uniform-ub}
There exists an $(\eps, \delta)$-PAC learning algorithm for $\cH_{n,s,k}^{W, B}$ with respect to the uniform distribution over $\cX$ that has sample complexity and run-time $O(n^d k^2 (W\sqrt{n} + B)^2 \log(1/\delta) / \eps)$ for $d = \Theta((k^8 W^4 \log (ns)^4 + k^6 B^4 \log s) / \eps^2)$
\end{theorem}

At a high level, we show that all hypotheses in $\cH_{n,s,k}^{W, B}$ exhibit low-degree Fourier concentration and hence can be learned over the uniform distribution using the low-degree algorithm~(\Cref{lem:low-degree-alg}). To show Fourier concentration, we bound the noise sensitivity of sparse-activated networks by first showing a bound on the average sensitivity and then converting this to a bound on noise sensitivity.
\begin{lemma}\label{lem:avg-sens-generalk}
For all $h \in \cH_{n,s,k}^{W, B}$, it holds that
$\AS(h) ~\le~ O\prn{k^4 W^2 \sqrt{n} \log(ns) + k^3 B^2 \sqrt{\log s}}$.
\end{lemma}
\begin{proof}
Consider $h \in \cH_{n,s,k}^{W, B}$ given as $h(x) = \sum_{j=1}^s u_j \relu(\ang{w_j, x} - b_j)$. For any $R \subseteq [s]$, let $\ell_R(x) ~=~ \ang{w^R, x} - b^R$ for $w^R :=\sum_{j \in R} u_j w_j$ and $b^R := \sum_{j \in R} u_j b_j$.
Since $\max_{j} |u_j| \cdot \max_j \|w_j\| \le W$ and $\max_j |u_j| \cdot \max_j |b_j| \le B$, it follows that $\|w^R\| \le |R| \cdot W$ and $|b^R| \le |R| \cdot B$.
For any $x \in \cX$, let $R_x \subseteq [s]$ be defined as $R_x := \{j \in [s] : \ang{w_j, x} > b_j\}$. Since $h$ is $k$-sparse, we have that $|R_x| \le k$ and hence $\|w^{R_x}\| \le k W$ and $|b^{R_x}| \le k B$.
It is easy to see that for $h \in \cH_{n,s,k}^{W, B}$ it holds that $h(x) = \ell_{R_x}(x)$ for all $x \in \cX$.

The average sensitivity of $h$ is given as
\begin{align*}
\AS(h)
&\textstyle~=~ \E_{x} \sq{\sum_{i=1}^n \frac14 \prn{h(x) - h(x^{\oplus i}}^2}\\
&\textstyle~=~ \E_{x} \sq{\sum_{i=1}^n \frac14 \prn{h(x) - h(x^{\oplus i})}^2 \cdot \mathds{1}\{R_x = R_{x^{\oplus i}}\}}\tag{U}\label{eq:same-partition}\\
&\textstyle~~~~+~ \E_{x} \sq{\sum_{i=1}^n \frac14 \prn{h(x) - h(x^{\oplus i}}^2 \cdot \mathds{1}\{R_x \ne R_{x^{\oplus i}}\}}\tag{V}\label{eq:diff-partition}
\end{align*}
We bound term \eqref{eq:same-partition} as,
\begin{align*}
\text{\eqref{eq:same-partition}}
&\textstyle~=~ \E_{x} \sq{\sum_{i=1}^n \frac14 \prn{\ell_{R_x}(x) - \ell_{R_x}(x^{\oplus i})}^2 \cdot \mathds{1}\{R_x = R_{x^{\oplus i}}\}}\\
&\textstyle~\le~ \E_{x} \sq{\sum_{i=1}^n \frac14 \prn{w^{R_x}_i}^2}
~=~ \E_{x} \sq{\frac14 \|w^{R_x}\|^2}
~\le~ \frac{k^2W^2}{4}\,.
\end{align*}
We bound term \eqref{eq:diff-partition} as follows using the inequality $(a-b)^2 \le 2(w^2 + b^2)$,
\begin{align*}
\text{\eqref{eq:diff-partition}}
&\textstyle~=~ n \E_{x, i} \sq{\frac14 \prn{h(x) - h(x^{\oplus i}}^2 \cdot \mathds{1}\{R_x \ne R_{x^{\oplus i}}\}}\\
&\textstyle~\le~ n \E_{x, i} \sq{\frac12 \prn{h(x)^2 + h(x^{\oplus i})^2} \cdot \mathds{1}\{R_x \ne R_{x^{\oplus i}}\}}\\
&\textstyle~=~ n \E_{x, i} \sq{h(x)^2 \cdot \mathds{1}\{R_x \ne R_{x^{\oplus i}}\}} \qquad \text{(by symmetry)}
\end{align*}
For $g_j(x) := \mathds{1}\{\ang{w_j, x} > b_j\}$, we have that
\begin{align*}
\Pr_{x, i} [R_x \ne R_{x^{\oplus i}}] 
&\textstyle~\le~ \frac{1}{n} \sum_{j=1}^s \sum_{i=1}^n \Pr_x[g_j(x) \ne g_j(x^{\oplus i})] = \frac{1}{n} \sum_{j=1}^s \AS(g_j)
\end{align*}
Note that $\sum_{j=1}^s g_j(x) \le k$ (by $k$-sparsity), and hence for $p_j = \E_x [g_j(x)]$, we have that $\sum_{j=1}^s p_j \le k$. From \Cref{lm:ashalfspace}, we have that $\AS(g_j) \le p_j \sqrt{n \log(1/p_j)}$. Thus,
\begin{align*}
\Pr_{x, i} [R_x \ne R_{x^{\oplus i}}] &~\le~ \frac1n \sum_{j=1}^s p_j \sqrt{n \log(1/p_j)} ~\le~ \frac{k \sqrt{\log(s/k)}}{\sqrt{n}}
\end{align*}
where we use concavity of $p \sqrt{\log(1/p)}$ for $p \in (0, 1)$.
For each $R \subseteq [s]$ with $|S| \le k$, we have by Hoeffding bound that for some sufficiently large $c$ and $t = ckW\sqrt{\log(n^k s)} + kB$,
\begin{align*}
    &\Pr_{x\sim \cU}\sq{\exists R \subseteq [s]: |R| \le k \ \text{ and } \ \left|\ang{w^R, x} - b^R\right| > t}\\
    &~\le~ \Pr_{x\sim \cU}\sq{\exists R \subseteq [s]: |R| \le k \ \text{ and } \ \left|\ang{w^R, x}\right| > t - \left|b^R\right|}\\
    &~\le~ 2 n^k \exp\prn{\frac{-(t-|b^R|)^2}{2\|w^R\|^2}} \le \frac{1}{(ns)^4},
\end{align*}
Hence, in particular we have that
\begin{align*}
\Pr_{x}[|\ell_{R_x}(x)| \ge ck^{1.5}W\sqrt{\log (ns)} + kB] &~\le~ \frac{1}{n^{4} s^4}
\end{align*}
And for all $x$, we also have that $|\ell_{R_x}(x)| \le kW\sqrt{n} + kB$ holds with probability $1$. Thus, we can upper bound \eqref{eq:diff-partition} as,
\begin{align*}
\text{\eqref{eq:diff-partition}}
&\textstyle~\le~ n \cdot \sq{\prn{\frac{k \sqrt{\log(s/k)}}{\sqrt{n}} - \frac{1}{(ns)^4}} (c k^{1.5} W \sqrt{\log(ns)} + k B)^2 + \frac{1}{(ns)^4} \cdot (kW\sqrt{n} + kB)^2}\\
&~\le~ O\prn{k^4 W^2 \sqrt{n} \log(ns) + k^3 B^2\sqrt{\log s}}
\end{align*}
Combining the bounds on \eqref{eq:same-partition} and \eqref{eq:diff-partition} completes the proof.
\end{proof}

Next, we can use the bound on average sensitivity to bound the noise sensitivity of functions in $\cH_{n,s,k}^{W, B}$. To do so we use an argument attributed to Peres for converting bounds on average sensitivity to bounds on noise sensitivity, allowing us to get better low-degree approximations.
\begin{lemma}\label{lem:as-to-ns-generalk}
For any $h \in \cH_{n,s,k}^B$,
\begin{align*}
\NS_{\rho}(h) = \sqrt{(1-\rho)} \cdot O(k^4 W^2 \log^2(ns/(1-\rho)) + k^3 B^2 \sqrt{\log s}).
\end{align*}
\end{lemma}
The proof of \Cref{lem:as-to-ns-generalk} is provided in \appendixref{apx:as-to-ns}.

\begin{proofof}[\theoremref{thm:generalk-uniform-ub}]
We combine \Cref{fact:fourier}, \Cref{lem:low-degree-alg} and \Cref{lem:as-to-ns-generalk}. Fix an error parameter $\eps$. 
Then, by \Cref{lem:as-to-ns-generalk}, there is a constant $c > 0$, such that for
\[
1-\rho = c \eps^2 \cdot \min\left\{\frac{\log(knsW^2/\eps)}{k^8 W^4 \log^4(ns)}, \frac{1}{k^6 B^4 \log s}\right\}
\]
any $h \in \cH_{n,s,k}^{W, B}$, satisfies
$$\NS_\rho(h) \leq \eps/3$$

Thus, we can choose a suitable $d = \Theta((k^8 W^4 \log (ns)^4 + k^6 B^4 \log s) / \eps^2)$, such that by \Cref{fact:fourier}, 
$$\textstyle
\sum_{T: |T| > d} \hat{f}(T)^2 \leq \frac{\eps}{3 (1-\rho^d)} \leq \frac{\eps}{d(1-\rho)} \le \eps\,.
$$
Finally, note that $C_{\cH_{n,s,k}^{W,B}} = k(W \sqrt{n} + B)$; since at most $k$ neurons are active on any input, and each neuron can at most contribute $W \sqrt{n} + B$. Thus, the theorem now follows from combining the above with \Cref{lem:low-degree-alg}. The run-time and sample complexity will be $O(n^d \log(1/\delta)/\eps)$ where $d$ is as above. 
\end{proofof}

\begin{remark}\label{rem:sparse-with-high-prob}
\Cref{thm:generalk-uniform-ub} can be extended to hold in case of the hypothesis class where $k$-sparsity need not hold for all inputs $x \in \cX$, but holds with probability at least $1 - 1/\poly(n, s)$ over the input distribution, that is, $\Pr_{x \sim \cU}[\#\{i \in [s] : \ang{w_i, x} + b_i > 0 \} > k] \le 1/\poly(n, s)$. This is by decomposing $\AS(h)$ into \eqref{eq:same-partition}, \eqref{eq:diff-partition} and a third term handling $x$ for which the $k$-sparsity is violated.
\end{remark}
\section{\boldmath Lower Bounds for Learning \texorpdfstring{$\cH_{n,s,1}$}{H n,s,1}}\label{sec:lb-uniform}

Note that the previous section implies a quasi-polynomial time learning algorithm for the class $\cH_{n,s,1}$ of $1$-sparsely activated networks. We next show that a quasi-polynomial run-time is likely necessary for learning $\cH_{n,s,1}$ under the uniform distribution and stronger lower bounds under arbitrary distributions. 
\paragraph{Sparse Activations Can Simulate Juntas} We first show that our proposed learning algorithms for the case of the uniform distribution have near-optimal runtime under a widely believed conjecture on the hardness of learning juntas. Let $\cJ_{n,p}$ denote the set of Boolean functions $f:\{1,-1\}^n \to \{-1,1\}$ that only depend on at most $p$ variables. 
\begin{conjecture}[Hardness of learning Juntas] (see e.g. \cite{mossel03juntas, feldman11lower}) \label{conj:junta-hardness}
There is no $(\eps, \delta)$-PAC learning algorithm for learning $\cJ_{n,p}$ under the uniform distribution on the hypercube that runs in time $n^{o(p)}$.\pritish{We cited \cite{mossel03juntas} and \cite{feldman11lower} in rebuttal. Cite those properly and note what they exactly say.}
\nishanth{problem is these citations also don't formally state this as a conjecture IIRC.}
\end{conjecture}

\noindent The conjecture implies that there is no learning algorithm for $\cH_{n,s,1}$ that runs in $n^{o(\log s)}$ time.

\begin{theorem}\label{thm:k=1-uniform-lb}
Assuming \Cref{conj:junta-hardness}, there is no $(\eps, \delta)$-PAC learning algorithm for $\cH_{n,s,1}^{W, B}$ for $W = \sqrt{\log_2 s}$ and $B = \log_2 s$ over $\cU$ that runs in $n^{o(\log s)}$ time.
\end{theorem}
\begin{proof}
We show that $\cH_{n,s,1}^{\sqrt{p}, p} \supseteq \cJ_{n, p}$ for all $p \le \floor{\log_2 s}$, that is, for $p \le \floor{\log_2 s}$ any $p$-junta $f \in \cJ_{n,p}$ can be expressed as $\sum_{j \in [s]} u_j \relu(\ang{w_j, x} + b_j)$ where $\|u\|_\infty \le 1$ and $\|w_j\|_2 \le \sqrt{\log_2 s}$.
Suppose w.l.o.g. that $f$ depends on $x_1, \ldots, x_p$. Let $w_1, \ldots w_{2^p}$ be distinct vectors that take all possible $\pm 1$ values in the first $p$ coordinates, and are $0$ on other coordinates. Let $u_j = f(x)$ for any $x$ such that $x_i = w_{ji}$ for all $i \in [p]$ and $j \in [2^p]$. Let $w_j = \mathbf{0}$ and $u_j = 0$ for all $j > 2^p$. It is now easy to verify that for all $x \in \cX$,
\[\textstyle
f(x) = \sum_{j \in [2^p]} u_j \relu(\ang{w_j, x} - p + 1), \quad \text{since } \relu(\ang{w_j, x} - p + 1) = \mathds{1}\{x_i = w_{ji} \text{ for all } i \in [p]\}
\]
Thus, the theorem follows under the assumption of \Cref{conj:junta-hardness}.
\end{proof}

\paragraph{Hardness Under Arbitrary Distributions} We next show that one-sparse activation networks over $\{1,-1\}^n$ can simulate parities of size $\Omega(\sqrt{n})$. Fix an integer $m$, and for $S \subseteq [m]$, let $\chi_S:\{1,-1\}^m \rightarrow \{0,1\}$ be defined by $\chi_S(y) = 1$ if and only if $\sum_{i\in S} y_i$ is even. Now, we can use the following simple identity (similar identities were used for similar purposes for example in \cite{KlivansS06})
$$\textstyle
\chi_S(y) = \sum_{a \in \{-m,\ldots,m\}: a\; even} 2 \relu\left( \frac{1}{2} - \left(\sum_{i \in S} y_i - a\right)^2\right). $$

Note that for any $y \in\{1,-1\}^m$, at most one ReLU node is active. This is not quite enough to capture $\cH_{n,s,1}$ as the function inside the ReLUs are not linear. To fix this, we linearize the quadratic function by increasing the dimension. For $y \in \{1,-1\}^m$, let $x(y) \in \{1,-1\}^{m\times m}$ be defined as follows:
$$ x(y)_{ij} = \begin{cases}
    y_i & \text{if $i=j$}\\
    y_i y_j &\text{if $i \neq j$}
\end{cases}.$$
 Let $n = m^2$ and identify $\{1,-1\}^n$ with $\{1,-1\}^{m \times m}$ in the natural way. Observe that for any $S \subseteq [m]$, $a \in [-m,m]$, there exists a vector $w_{S,a} \in \R^n, b_{S,a} \in \R$ such that 
$$\textstyle
\frac{1}{2} - \left (\sum_{i\in S} y_i -a \right)^2 = \langle w_{S,a},x(y)\rangle - b_{S,a}.
$$

In particular, we can take $b_{S,a} = |S| + a^2 - 1/2$, and $w_{S,a}[i,j] = -1$ if $i \neq j \in [m]$ and $w_{S,a}[i,i] = 2a$. Note that $\|w_{S,a}\|_2 = O(m^{1.5}) = O( n^{3/4})$ and $|b_{S,a}| = O(m^2) = O(n)$. 

In summary, there exists a distribution $\cal{D}$ on $\{1,-1\}^{m \times m}$ such that learning parities over $\{1,-1\}^m$ under the uniform distribution is implied by learning $\cH_{m^2,2m,1}^{O(m^{1.5}), O(m^2)}$ under the distribution $\cal{D}$. The first part of \Cref{th:lb-general} now follows from standard lower bounds for learning parities. 

\paragraph{SQ Hardness} Consider a class of functions, denoted by $C$, that maps $\mathbb{R}^n$ to $\mathbb{R}$, and let $\cD$ be a distribution over $\mathbb{R}^n$.

In the Statistical Query (SQ) model, as described by \cite{Kearns98}, the learner interacts with the data through an SQ oracle. For a bounded query function $\phi: \mathbb{R}^n \times \mathbb{R} \rightarrow [-1, 1]$ and a tolerance $\tau > 0$, the oracle can return any value $v$ such that the absolute difference $\left|v - \mathbb{E}_{x \sim D}[\phi(x, f(x))]\right| \leq \tau$. The goal in SQ learning is to learn an approximation to the unknown concept only using few queries as above with reasonable tolerance. We will use the following classical theorem:

\begin{theorem}[\citep{blum1994weakly}]
Any SQ algorithm for learning the class of parities over $\{1,-1\}^m$ within error $1/3$ under the uniform distribution over the hypercube with tolerance $\tau$ requires $\Omega(2^m \tau^2)$ queries. 
\end{theorem}

The first part of \Cref{th:lb-general} follows immediately from the above and the fact that parities on $m$ variables can be computed in $\cH_{m^2, O(m),1}^{O(m^{1.5}), O(m^2)}$ as described. 

\paragraph{Cryptographic Hardness} We sketch the argument here. Following \cite{ChenGKM22}, our starting point will be the \emph{Learning with Rounding} (LWR) problem~\citep{banerjee2012pseudorandom}:

\begin{definition}
    For moduli $p,q \in \mathbb{N}$, $w \in \mathbb{Z}_q^m$, let $f_w:\mathbb{Z}_q^m \rightarrow \mathbb{Z}_p$ by 
    $f_w(y) := (\langle w,y\rangle \text{ mod } q) \mod p.$
\end{definition}

In the $\mathsf{LWR}_{p,q,m}$ problem the \emph{secret} $w \in \mathbb{Z}_q^m$ is drawn uniformly at random and we are given samples of the form $(y,f_w(y))$ where $y$ is uniform over $Z_q^m$. The goal is to output a hypothesis that achieves a small error in predicting the label $f_w(\cdot)$. It is conjectured that there is no $\poly(m,p,q)$ algorithm for $\mathsf{LWR}_{p,q,m}$.

\begin{conjecture}[See \cite{banerjee2012pseudorandom}]
    There is no $\poly(p,q,m)$ run-time algorithm to solve the $\mathsf{LWR}_{p,q,m}$ with probability at least $2/3$ (over the random choice of $w$ and the samples).
\end{conjecture}

We show that an efficient algorithm for $\cH_{n,s,1}$ functions under arbitrary distributions on the hypercube will contradict this assumption. 

Consider an instance of the $\mathsf{LWR}_{p,q,m}$ problem. First, map $y \in \mathbb{Z}_q^m$ to $ z(y) \in \{1,-1\}^r$ for $r= O(m \log q)$ by considering the binary representation of the integers in $y$. Next, let $\lambda: [q^2 m] \rightarrow [p]$ be such that $\lambda(i) = ( i \mod q) \mod p$. Note that for every $w \in \mathbb{Z}_q^m$, we can find a vector $v(w) \in \mathbb{R}^r$ such that $\langle v(w), z(y)\rangle = \langle w, y\rangle$. Then,
$$f_w(y) = \lambda(\langle v(w), z(y) \rangle).$$

Now, observe that we can write 
$$\textstyle
\lambda(\langle v(w), z(y) \rangle) = \sum_{a \in [q^2 m]} 2 \lambda(a) \relu\left(\frac{1}{2} - \left( \langle v(w), z(y) \rangle - a)\right)^2\right).$$

Note that in the conversion $z(y) \in \{1,-1\}^r$ and $v(w) \in \mathbb{R}^r$. Further, for any input $y$, only one of the ReLUs will be active. However, the above is not quite in $\cH_{n,s,1}$ as we have a quadratic function inside the ReLU. Just as we did for parities, we can fix this issue by linearizing the quadratic form. Let $n = r^2$, and define $x(y) \in \{1,-1\}^{r \times r}$ by setting $x(y)_{ij} = z(y)_i z(y)_j$ if $i \neq j$ and $x(y)_{ii} = z(y)_i$. Then, just as in our argument for parities, there exists a \emph{lifted} weight vector $W_{w,a} \in\{1,-1\}^n$ and $b_{w,a}$ such that 

$$\frac{1}{2} - \left( \langle v(w), z(y) \rangle - a)\right)^2 = \langle W_{w,a}, x(y) \rangle - b_{w,a}.$$

In addition, it is easy to check that $\|W_{w,a}\|_2, |b_{w,a}| = \poly(q,m)$. 
In particular, we get that for every $w \in \mathbb{Z}_q^m$, there exists a function $F_w$ in $\cH_{r^2, O(q^2 m),1}$ such that for every $y \in \mathbb{Z}_q^m$, 
$$f_w(y) = F_w(x(y)),$$
where $x(y) \in \{1,-1\}^{r^2}$ is the embedding as defined above and in showing SQ hardness. The second part of \Cref{th:lb-general} now follows from the conjectured hardness of $\mathsf{LWR}_{p,q,m}$; we omit the minor details.

\section{Learning under General Distributions}
\label{sec:rademacher}

We now show the statistical advantage associated with sparsely activated neural networks over general distributions. In particular, we show that

\begin{theorem}\label{thm:general-dist-upper-bound}
There exists a $(\eps, \delta)$-PAC learning algorithm for any $\cH_{n,s,k}^{W,B}$ with sample complexity $m(\eps, \delta) = O \Big(\frac{(WR+B)^2 k s n \log(\frac{k(R+B)}{\eps}) + \log(\frac 1 \delta)}{\eps^2} \Big)$.
\end{theorem}

This result even holds in a more general setting where the input space
$\cX \subset \R^n$ and $\|x\| \leq R$ for all $x \in \cX$. To begin with we will again consider the class of $1$-sparsely activated networks, i.e., $\cH^{W,B}_{n,s,1}$. We will discuss extensions to $\cH^{W,B}_{n,s,k}$ towards the end of the section.

We use Rademacher complexity to establish the bound in Theorem~\ref{thm:general-dist-upper-bound}. Given a set of examples $S = \{x_1, x_2, \ldots, x_m\}$ the empirical Rademacher complexity \citep{shalev2014understanding} is defined as $\cR_{\cH}(S) \textstyle~:=~ \mathbb{E}_\zeta \sq{\max_{h \in \cH} \frac{1}{m} \sum_{i=1}^m \zeta_i h(x_i)}$,
where $\zeta_1, \ldots, \zeta_m$ are $\{-1,+1\}$ valued Rademacher random variables.
For $\cH$, let $C_{\cH} := \sup_{h \in \cH, x \in \cX} h(x)$.

\begin{lemma}[see \cite{shalev2014understanding}]\label{lem:rademacher-to-sample-complexity}
For any class $\cH$ mapping $\cX$ to $\R$, there exists an $(\eps, \delta)$-PAC learning algorithm for $\cH$ with sample complexity $m(\eps, \delta)$ equal to the smallest $m$ such that for a large enough constant $c$, it holds that\pritish{I am combining Lemmas 26.5 and 26.9 from \cite{shalev2014understanding}, and using that square loss is $O(C_\cH)$-Lipschitz.}
\[\textstyle
c \cdot \prn{C_{\cH} \mathbb{E}_S [\cR_{\cH}(S)] + \sqrt{\frac{\log(1/\delta)}{m}}} \le \eps\,.
\]
\end{lemma}

\Cref{thm:general-dist-upper-bound} will follow from bounding the Rademacher complexity $\cR_{\cH}(S)$.
%
Recall that in the absence of any sparsity assumption, existing results \citep{anthony1999neural} on the Rademacher complexity of $1$-hidden layer ReLU networks with input dimensionality $n$ and $s$ hidden units lead to a bound of $\frac{(WR+B) {s}}{\sqrt{m}}$.\footnote{Better bounds are possible under stronger assumptions on the network weights \citep{wei2018margin}.}\pritish{Double check that these bounds apply under our updated scaling.} We will show that the main statistical advantage that comes from sparsity is that the dependence on the number of hidden units $s$ can be made sub-linear, albeit at the expense of an explicit dependence on the input dimensionality $n$. In particular we will prove the following theorem.

\begin{theorem}\label{thm:rademacher-1-sparse-bound-2}
  It holds that
  \begin{align}
      \cR_{\cH^{W,B}_{n,s,1}}(S) &
      ~\leq~ \frac{(WR+B)\sqrt{sn \log(m (R+B))}}{\sqrt{m}}.
  \end{align}
\end{theorem}
\begin{proof}
    For a given hypothesis $u,w_1, \ldots, w_s \in \cH^{W,B}_{n,s,1}$ and for any $j \in [s]$, let $I_j$ be the subset of the $m$ examples that activate neuron $j$, i.e., $I_j = \{i \in [m]: \ang{w_j, x_i} - b_j \geq 0\}$. Since each $I_j$ is determined by a halfspace in $n$ dimensions, by the Sauer-Shelah lemma~\citep{shalev2014understanding} there can be at most $O(m^n)$ such subsets.

Next, we have
\begin{align}
      \cR_{\cH^{W,B}_{n,s,1}}(S) &\textstyle~:=~ \mathbb{E}_\zeta \sq{\max_{u, w_1, \ldots, w_s \in \cH^{W,B}_{n,s,1}} \frac{1}{m} \sum_{i=1}^m \zeta_i \sum_{j=1}^s u_j \relu(\ang{w_j, x_i} - b_j)}\\
      &\textstyle~=~ \mathbb{E}_\zeta \sq{\max_{u, w_1, \ldots, w_s \in \cH^{W,B}_{n,s,1}} \frac{1}{m} \sum_{j=1}^s \sum_{i \in I_j} \zeta_i u_j (\ang{w_j, x_i} - b_j)}\\
      &\textstyle~\leq~ \mathbb{E}_\zeta \sq{\max_{u, w_1, \ldots, w_s \in \cH^{W,B}_{n,s,1}} \frac{1}{m} \sum_{j=1}^s \sum_{i \in I_j} \zeta_i u_j \ang{w_j, x_i}}\nonumber\\
      &\textstyle~~~~+~ \mathbb{E}_\zeta \sq{\max_{u, w_1, \ldots, w_s \in \cH^{W,B}_{n,s,1}} \frac{1}{m} \sum_{j=1}^s \sum_{x_i \in I_j} \zeta_i u_j b_j}
\end{align}

We will bound the above two terms separately via standard concentration inequalities. For the second term note that for any fixed $I_j$, the random variable $\sum_{i \in I_j} \zeta_i$ is sub-Gaussian with norm $O(\sqrt{|I_j|})$. Hence we for any fixed $I_j$ the following holds \citep{vershynin2018high}
\begin{align}\textstyle
    \mathbb{P} \left[\left|\sum_{i \in I_j} \zeta_i\right| > t\sqrt{|I_j|} \right] \leq 2 e^{-\frac{t^2}{c }},
\end{align}
where $c > 0$ is an absolute constant. Via the union bound we get that with probability at least $1-O(m^n e^{-t^2/c})$, all sets $I_j$ simultaneously satisfy the above inequality. 

Hence we get the following bound on the second term.
\begin{align}
\mathbb{E}_\zeta \sq{\max_{u, w_1, \ldots, w_s \in \cH^{W,B}_{n,s,1}} \frac{1}{m} \sum_{j=1}^s \sum_{x_i \in I_j} \zeta_i u_j b_j} &~\leq~ \frac{1}{m} \sum_{j=1}^s t|u_j b_j| \sqrt{|I_j|} + O(m^n e^{-t^2/c}) \frac{1}{m} \sum_{j=1}^s |u_j b_j| |I_j|.
\end{align}

From the fact that the activations are $1$-sparse we get that $\sum_{j=1}^s|I_j| = m$. This implies that $\sum_{j=1}^s \sqrt{|I_j|} \le \sqrt{sm}$. Furthermore, using the fact that $\max_j |u_j b_j| \leq B$ we get
\begin{align}
    \mathbb{E}_\zeta \sq{\max_{u, w_1, \ldots, w_s \in \cH^{W,B}_{n,s,1}} \frac{1}{{m}} \sum_{j=1}^s \sum_{x_i \in I_j} \zeta_i u_j b_j} &~\leq~ \frac{1}{\sqrt{m}} t B \sqrt{s} + O(m^n e^{-t^2/c}) B.
\end{align}
Setting $t = 2\sqrt{nc \log(mB)}$ we get that the second term is bounded by
\begin{align}
    \mathbb{E}_\zeta \sq{\max_{u, w_1, \ldots, w_s \in \cH^{W,B}_{n,s,1}} \frac{1}{m} \sum_{j=1}^s \sum_{x_i \in I_j} \zeta_i u_j b_j} &~\leq~ \frac{4B \sqrt{s n c \log(m B)}}{\sqrt{m}}.
\end{align}

Similarly, we next bound the first term. Note that for any fixed $I_j$, and any coordinate $p \in [n]$, sub-Gaussian concentration~\citep{vershynin2018high} implies that
\begin{align}
    \mathbb{P} \left(\Big|\sum_{i \in I_j} \zeta_i x_{i,p}\Big| > t\sqrt{\sum_{i \in I_j} x^2_{i,p}} \right) \leq 2 e^{-\frac{t^2}{c}}.
\end{align}
\end{proof}
Via a union bound over all the $n$ coordinates and all possible subsets $I_j$ we get that with probability at least $1-2n m^n e^{-\frac{t^2}{c}}$, all sets $I_j$ simultaneously satisfy
\begin{align}
    \left\|\sum_{i \in I_j} \zeta_i x_i\right\| \leq tR\sqrt{|I_j|}.
\end{align}

Using the above we can bound the first term as
\begin{align}
    \mathbb{E}_\zeta \sq{\max_{u, w_1, \ldots, w_s \in \cH^{W,B}_{n,s,1}} \frac{1}{m} \sum_{j=1}^s \sum_{x_i \in I_j} \zeta_i u_j \ang{w_j, x_i}} &~\leq~ \mathbb{E}_\zeta \sq{\max_{u, w_1, \ldots, w_s \in \cH^{W,B}_{n,s,1}} \frac{1}{m} \sum_{j=1}^s \ang{u_j w_j, \sum_{i \in I_j} \zeta_i x_i}}\\
    &~\leq~ \frac{1}{m} \sum_{j=1}^s |u_j|\|w_j\| \mathbb{E}_\zeta \sq{\left\|\sum_{i \in I_j} \zeta_i x_i\right\|}\\
    &~\leq~ \frac{W}{m} \sum_{j=1}^s \big( tR\sqrt{ |I_j|} + 2n m^n e^{-\frac{t^2}{c}} R |I_j| \big).
\end{align}

Recall from above that $\sum_{j=1}^s \sqrt{|I_j|} \le \sqrt{sm}$. Furthermore, setting $t =2\sqrt{n \log(mR)}$ we get that the first term is bounded by 
\begin{align}
    \mathbb{E}_\zeta \sq{\max_{u, w_1, \ldots, w_s \in \cH^{W,B}_{n,s,1}} \frac{1}{m} \sum_{j=1}^s \sum_{i \in I_j} \zeta_i u_j \ang{w_j, x_i}} &~\leq~ 4 \frac{WR \sqrt{n s  \log(mR)}}{\sqrt{m}}.
\end{align}

Combining the bounds for the first and the second terms, we get the desired claim.

\paragraph{\boldmath Generalization to $k$-sparsely activated networks.}
The above analysis extends in a straightforward manner to the class $\cH_{n,s,k}$, i.e., the class of networks where each input activates at most $k$ hidden units. 

To extend the bound in Theorem~\ref{thm:rademacher-1-sparse-bound-2} we note that using the fact that $k$-sparsity implies that $\sum_{j \in [s]} |I_j| \leq k m$ we get that
  \begin{align}
      \cR_{\cH^{W,B}_{n,s,k}}(S) &
      ~\leq~ \frac{(WR+B)\sqrt{snk \log(km (R+B))}}{\sqrt{m}}.
  \end{align}
Note that in contrast to the classical bounds on Rademacher complexity of general norm bounded $1$-layer neural networks the bound in Theorem~\ref{thm:rademacher-1-sparse-bound-2} above has a sub-linear dependence on $s$. 
However we incur an explicit dependency on the input dimensionality.

We suspect that this is a limitation of our proof technique and conjecture that the right dependence should not have any explicit dependence on the input dimension $n$.

\begin{conjecture}
    The class $\cH^{W,B}_{n,s,k}$ of $k$-sparsely activated neural networks satisfies
    \begin{align}
        \cR_{\cH^{W,B}_{n,s,k}}(S) &\textstyle \leq \frac{(WR+B) \sqrt{sk}}{\sqrt{m}}.
    \end{align}
\end{conjecture}

\section{Discussion \& Future Directions}\label{sec:conclusion}

Motivated by the empirical phenomenon of activation sparsity in MLP layers of large transformer models, in this work we proposed and studied the problem of PAC learning the class of sparsely activated neural networks. This is a novel concept class with many interesting properties. The form of input-dependent sparsity present in this class of functions makes it distinct from the typical sparse function classes studied in literature. The main conceptual insight from our work is that despite the empirical challenges in leveraging sparsity, activation sparsity can provably provide both computational and statistical benefits. 

Several open questions come out of our work. While we provide algorithms with near optimal running time for the case of the uniform distribution, it would be interesting to design learning algorithms under arbitrary distributions that are provably better than the $O((n s)^n)$-time algorithms that exist for general $1$-layer ReLU networks \citep{goel2020statistical}. As mentioned in \Cref{sec:rademacher} we strongly suspect that the dependence on the input dimension $n$ in the Rademacher complexity bound of \Cref{thm:rademacher-1-sparse-bound-2} is suboptimal. While we primarily considered networks that are sparsely activated for all inputs, it might be interesting to also consider sparsely activated with high probability over input distributions, as we briefly alluded to in \Cref{rem:sparse-with-high-prob} although in that case, the probability of not being sparsely activated was very small.
Finally, it would be interesting to explore practical algorithms for leveraging sparsity based on our theoretical insights.

\acks{We thank anonymous reviewers for their comments that helped improve the presentation.}

\newpage
\bibliography{refs}
\newpage

\appendix

\section{Example of a Sparsely Activated Network without Weight Sparsity}
\label{apx:san-without-deadsparsity}
There are interesting functions (beyond juntas/parities) that are sparsely activated but do not have weight sparsity. E.g.: suppose $\log_2 s < n$. Consider $b = \log_2 s$, $q=n-b$, and look at $F:\{-1,1\}^b \times \{-1,1\}^{q} \rightarrow \mathbb{R}$, of the form $\sum_{\alpha \in \{1,-1\}^b} \sigma(\langle w_{\alpha}, y \rangle + \Gamma \cdot ( \langle x, \alpha \rangle - b))$, where the input is $(x, y)$. When $\Gamma = \sqrt{q}$, this network is $1$-sparsely activated for all inputs, and when $\Gamma = \Theta(\sqrt{\log s})$, the function is $1$-sparse with probability $1-1/\poly(s)$ under the uniform distribution on $\{-1, 1\}^{b+q}$. Remark~\ref{rem:sparse-with-high-prob} shows that our results continue to hold in such a setting. Intuitively, such functions are similar to Indexing; they return the function $\sigma(\langle w_x, y\rangle)$ for all (or most) of the input space, where $w_x$ can depend arbitrarily on the $x$ part of the input.

\section{Proof of \Cref{lem:as-to-ns-generalk}}
\label{apx:as-to-ns}
\begin{proofof}[\Cref{lem:as-to-ns-generalk}]
    Given a $\rho \in [-1, 1]$, let $r = \lfloor 2/(1-\rho) \rfloor$. We describe an alternate way to sample $(x, N_\rho(x))$. First sample $z \in \{\pm 1\}^n$ uniformly at random and partition the $n$ coordinates of $z$ into the $r$ buckets $\{A_e \subseteq [n]\}_{e=1}^r$ at random (each coordinate is included in exactly one of these buckets uniformly and independently). For each $A_e$, sample $v_e \in \{\pm 1\}$ uniformly at random. Multiply the coordinates of $A_e$ by $v_e$ and concatenate all the buckets to get $x$. Choose one bucket $b$ at random and flip $v_b$ to get $v^{\oplus b}$. Multiply the coordinates of $A_e$ by $v^{\oplus b}$ to get $y$. Observe that $(x,y)$ are distributed exactly the same as $(x,N_\rho(x))$. 
 Now, given
 $h(x) = \sum_{j=1}^s u_j \relu(\ang{w_j, x} - b_j)$, define
 \begin{align*}
     H_z(v) = \sum_{j=1}^s u_j\relu(\ang{w_j', v}-b_j),
 \end{align*}
 where $w_{je}' = \sum_{l \in A_e} w_{jl}z_l$.
 Clearly $h(x) = H_z(v)$. Hence,
 \begin{align}\label{eq:astons}
     \NS_{\rho}(h) &= \E[(h(x) - h(y))^2]\nonumber\\
     &= \frac{1}{r}\E_{z,\{A_e\}}\left(\sum_{b=1}^r \E_v (H_z(v) - H_z(v^{\oplus b}))^2\right)\nonumber\\
     &=\frac{1}{r} \E_{z, \{A_e\}}[\AS(H_z)].
 \end{align}
 From \Cref{lem:avg-sens-generalk},
 \begin{align*}
     &\AS(H_z) \le O\left(k^4W'^2\sqrt{r}\log (rs) + k^3 B^2 \sqrt{\log s}\right),\\
     &\text{where } W' := \max_{j \in [s]} |u_j| \cdot \|w_j\|.
 \end{align*}
 To bound $W'$ we need to bound
 \begin{align*}\textstyle
     \max_{j \in [s]}\|w_j'\|^2_2 = \max_{j \in [s]} \sum_{e=1}^r\left(\sum_{i \in A_e}w_{ji}z_i \right)^2.
 \end{align*}
 For any $j \in [s]$, we have from measure concentration
\begin{align*}
    &\textstyle\Pr_z \left[\left|\sum_{i \in A_e} w_{ji}z_i\right| > t \right] \le 2\exp\left(-\frac{t^2}{4\sum_{i \in A_e} w_{ji}^2} \right) \\
    \implies &\textstyle\Pr_z \left[\left|\sum_{i \in A_e} w_{ji}z_i\right| > 2\sqrt{2\log(nsr)\sum_{i \in A_e} w_{ji}^2} \right] \le \frac{1}{(nsr)^2}
    \end{align*}
    Now we use that $\sum_{e=1}^r \sum_{i \in A_e} w_{ji}^2 = \|w_j\|_2^2$.
    \begin{align*}
    \implies &\textstyle\Pr_z \left[\sum_{e=1}^r\left(\sum_{i \in A_e} w_{ji}z_i\right)^2 > 8\log(nsr) \|w_j\|_2^2\right] \le \frac{1}{n^2s^2r} \\
    \implies &\Pr_z\left[\forall j \in [s], \; \|w'_j\|_2^2 \le 8\log(nsr) \|w_j\|_2^2\right] \ge 1-\frac{1}{n^2sr}.
\end{align*}
Combining with the fact that $\|w'_j\|$ is always at most $W \sqrt{n}$, we get that 
$$\E_{z}\left [\max_{j \in [s]} \|w'_j\|_2^2\right ] \leq O(\log ns r) \|w_j\|_2^2.$$
Combining the above with \eqref{eq:astons}, we get
\begin{align*}
    \NS_\rho(h) &= \frac{1}{r} \E_{z, \{A_e\}}[\AS(H_z)] \le \frac{O(W^2 k^4\log^2 (nrs)^2 + k^3 B^2 \sqrt{\log s})}{\sqrt{r}}\\ 
    &= \sqrt{(1-\rho)}O(k^4 W^2 \log^2(ns/(1-\rho)) + k^3 B^2 \sqrt{\log s}).
\end{align*}
 The claim now follows.
\end{proofof}

\end{document}